\documentclass[conference]{IEEEtran}
\usepackage{times}

\usepackage[numbers]{natbib}
\usepackage{multicol}
\usepackage[bookmarks=true]{hyperref}
\usepackage{amsmath} 
\usepackage{amssymb}  
\usepackage{subcaption}
\usepackage{graphicx} 
\usepackage{mathtools, nccmath}
\usepackage{algorithm}
\usepackage{algpseudocode}
\usepackage{tikz}
\usetikzlibrary {arrows.meta}
\usepackage{siunitx}

\newtheorem{theorem}{Theorem}[section]

\newtheorem{lemma}[theorem]{Lemma}
\newenvironment{proof}[1][Proof]{%
  \noindent\textbf{#1.}\ \ignorespaces
}{%
  \hfill$\square$\par\medskip
}
 
\newcommand\blfootnote[1]{%
  \begingroup
  \renewcommand\thefootnote{}%
  \NoHyper\footnote{#1}\endNoHyper%
  \addtocounter{footnote}{-1}%
  \endgroup
}
\DeclareUrlCommand\url{\color{purple}}

\begin{document}

\title{Sampling-Based Follow-the-Leader Motion Planning for Manipulator-Mounted Continuum Robots}


\author{
Chengnan Shentu$^{*}$, Nicholas Baldassini$^{*}$, Oluwagbotemi D. Iseoluwa, Radian Gondokaryono, Jessica Burgner-Kahrs\\
University of Toronto\\
\url{https://continuumroboticslab.github.io/sb-ftl-cr-planner/}
}


\twocolumn[{%
	\renewcommand\twocolumn[1][]{#1}%
	\maketitle
    \vspace{-4mm}
	\begin{center}
		\includegraphics[width=17.5cm, height=5.5cm]{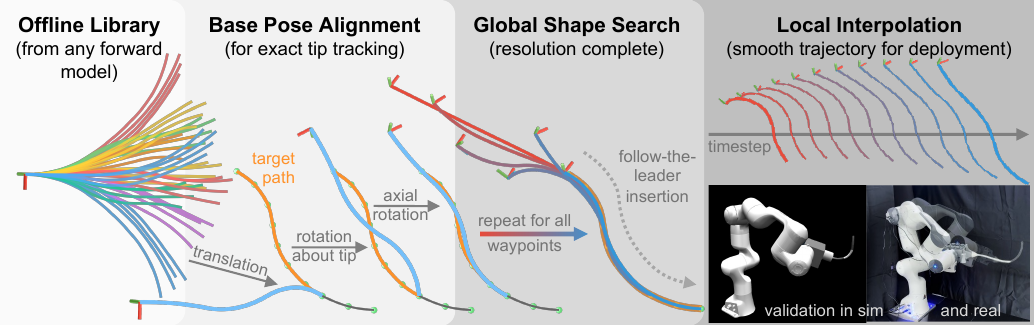}
		\captionof{figure}{\textbf{Overview} of our sampling-based follow-the-leader motion planning framework. 
        \textbf{Offline Library:} shapes are precomputed from any forward model. 
        \textbf{Base Pose Alignment:} for each candidate shape (blue), geometric transformations align the robot tip to the current waypoint (green dots) and shape along the target path (orange curve). Fully-actuated $SE(3)$ base poses are shown as coordinate axes. 
        \textbf{Global Shape Search:} the best-matching shape is selected for each waypoint, producing a sparse plan (red to blue gradient indicates progression). 
        \textbf{Local Interpolation:} smooth interpolation between waypoints yields a dense trajectory suitable for deployment, validated in simulation and on hardware.}
    \label{fig:teaser}
	\end{center}
}]

\IEEEpeerreviewmaketitle
\blfootnote{$*$ Indicates equal contribution. Correspondence to {\tt\small \{c.shentu, nicholas.baldassini\}@mail.utoronto.ca}.}

\begin{abstract}
Follow-the-leader (FTL) motion exploits the unique morphology of continuum robots (CRs) to navigate confined spaces by having the body retrace the path of the tip. 
While extensively studied, existing FTL methods typically assume a fixed base or a single degree-of-freedom insertion mechanism, limiting their applicability to practical systems in which CRs are mounted on robotic manipulators with fully actuated $SE(3)$ base pose.
This paper presents a sampling-based motion planner for FTL motion of manipulator-mounted CRs that jointly considers robot configuration and base pose. The key idea is to decouple global shape search from base pose determination by computing the base pose through a closed-form geometric construction, thereby avoiding iterative optimization during online planning. The approach supports general forward models and enables efficient planning by shifting the majority of computation offline.
We establish theoretical guarantees including resolution complete shape search and converging tip tracking throughout waypoint traversal and interpolation.
Experiments on 120 simulated paths over 3 test classes demonstrate 0\% tip error and  1.9\% mean shape deviation (w.r.t. robot length) at 100\% success rate. 
We validate the practicality of our approach on a 6-DOF tendon-driven CR mounted on a serial manipulator.
Code and visualization available at \href{https://continuumroboticslab.github.io/sb-ftl-cr-planner/}{the project website}.
\end{abstract}

\section{Introduction}
\label{sec:intro}

Continuum robots (CRs) are flexible, high-DOF robots inspired by biological structures such as elephant trunks and octopus arms, with promising potential for navigating confined, tortuous spaces inaccessible to rigid-link robots~\cite{robinson_1999_continuum}. 
Their unique capabilities have led to increasing deployment in minimally invasive surgery~\cite{burgner_2015_continuum_robots_medical, dupont_2022_continuum_robots_medical}, industrial inspection~\cite{wang_2021_deisgn_modelling_validation} and search-and-rescue operations~\cite{yamauchi_2022_development}. 
A key motion primitive enabling these applications is follow-the-leader (FTL), where the robot body follows the path traced by its tip during insertion~\cite{neumann_2016_considerations_ftl}. 
FTL motion minimizes lateral forces on surrounding tissue or structures, reducing trauma in surgical procedures and avoiding collisions in cluttered environments.

FTL motion planning has been studied extensively for various continuum and soft robot architectures. 
Early work established the foundational concepts for snake-like robots~\cite{choset_1999_ftl_approach}. 
Steerable needles~\cite{cowan_2010_robotic_needle_steering} and vine robots~\cite{coad_2020_vine_robots} are designed to operate with FTL motion, achieving excellent results because they naturally extend during deployment. 
Similar principles have been applied to concentric tube robots (CTRs), where FTL is achieved through selection of suitable tube parameters and coordinated tube extension~\cite{gilbert_2015_ctr_steerable_needles, garriga_2018_complete_ftl_kinematics, girerd_2020_slam_ftl, xu_2024_shape_control_ctr}.

For tendon-driven continuum robots (TDCRs), two main approaches have emerged. 
The first adds degrees of freedom to the TDCR design itself through extensible segments~\cite{amanov_2021_tdcr_extensible_sections}, active stiffness modulation~\cite{qian_2025_jammingsnake}, magnetic actuation~\cite{mao_2024_magnetic_steering}, or a combination such as the interlaced design with extensible segments and internal locking mechanism~\cite{kang_2016_first_interlaced_cr, wang_2021_ftl_deployment_interlaced}. 
While effective, these hardware-based solutions are difficult to generalize across applications as the additional mechanisms occupy valuable space that could otherwise accommodate payloads such as sensors, tools, or delivery channels.

The second approach addresses the more general case by introducing degrees of freedom at the TDCR base and coordinating them with the CR's configuration through differential~\cite{gao_2020_cr_ftl_edoscopic}, heuristic-based~\cite{mohammad_2021_efficient_ftl, ma_2025_geometric_iterative}, or optimization-based inverse kinematics~\cite{palmer_2014_realtime_method_tip, gao_2021_ftl_motion_strategy, xiao_2026_mobile_cr}. 
However, these methods share common limitations. 
All rely on simplified kinematics such as the piece-wise constant curvature (PCC) model~\cite{jones_2006_kinematics_multisection_cr}, which does not accurately capture the behavior of physical TDCRs. 
Additionally, most employ local solvers or iterative optimization, which can converge to suboptimal solutions in the high-dimensional, nonlinear configuration space of CRs and are sensitive to initialization. 
Such unpredictability can be a limitation for safety-critical applications where predictable planner behavior is important.

These limitations are compounded by a confining assumption shared by nearly all existing methods: the robot base is either fixed in space or constrained to 1-DOF linear insertion. 
This does not reflect the growing class of practical deployments where continuum robots are mounted on robotic manipulators with full 6-DOF base pose control. 
In surgical systems, flexible endoscopes have been integrated with robot arms for precise positioning~\cite{wilkening_2017_concurrent_control, berthet_2018_i2snake,chen_2024_design}. 
In industrial inspection, snake-like robots have been mounted on serial manipulators~\cite{buckingham_2012_nuclear}, mobile platforms~\cite{xiao_2026_mobile_cr}, and even aerial vehicles~\cite{peng_2025_dexterous} to extend reach and dexterity.

In this paper, we consider the FTL motion planning problem for CRs whose base pose is fully actuated in SE(3), as arises when a CR is mounted on a serial manipulator. Given a spatial path specified by ordered waypoints, the objective is to generate a continuous motion such that (i) the robot tip tracks the path during insertion, and (ii) the inserted portion of the robot body follows the path with minimal deviation. 
While a fully actuated base offers greater flexibility and expanded motion capability, it also introduces computational challenges by expanding the planning space by six additional dimensions.

We address this problem with a sampling-based framework that decouples global shape search from local interpolation. The key enabler is a closed-form geometric construction that uniquely determines the 6-DOF base pose for any selected shape, guaranteeing exact tip placement without iterative optimization. This decoupling allows shape search to be performed via global search over a precomputed library—compatible with any forward model—while interpolation between waypoints exploits radial symmetry to ensure smooth transitions with a small number of forward model evaluations.

To summarize our contribution, we present a FTL motion planning framework for manipulator-mounted CRs with exact tip tracking by construction. The framework supports general forward models beyond PCC while enabling efficient online planning through user-adjustable clustering. We establish formal guarantees including resolution completeness for shape search and asymptotic convergence for tip tracking, and validate the approach in simulation and on hardware. The implementation is open-sourced to support reproducibility. This is the first FTL method for continuum robots with a fully-actuated $SE(3)$ base and formal planning guarantees.

\section{Problem Formulation}
\label{sec:problemFormulation}

Following the standard FTL problem formulation, the desired path is given as an ordered set of 3D waypoints
$\mathcal{W}=\{\mathbf{w}_1, \mathbf{w}_2, \ldots, \mathbf{w}_n\}$, $\mathbf{w}_i\in\mathbb{R}^3$.
Unlike methods that assume specific path representations (e.g., constant-curvature arcs), our formulation accepts arbitrary waypoint sequences, whether user-defined, discretized from continuous curves, or generated by higher-level planners.
The shape of a CR in configuration $\mathbf{q}$ is described by backbone points $\mathbf{p}(s)$ for arc length $s \in [0, S]$, computed via a forward model (FM), where $\mathbf{p}(0)$ is the base and $\mathbf{p}(S)$ is the tip.
In practice, we discretize this shape as $\mathbf{P} = [\mathbf{p}_0, \ldots, \mathbf{p}_D]^\top \in \mathbb{R}^{(D+1) \times 3}$, with $\mathbf{p}_0 = \mathbf{p}(0)$ and $\mathbf{p}_D = \mathbf{p}(S)$.

We consider a CR mounted to a 6-DOF base, representing practical systems where CRs are attached to serial manipulators.
The CR configuration space is denoted $\mathcal{Q} \subset \mathbb{R}^d$, and the combined system configuration space is
$\mathcal{K} = SE(3) \times \mathcal{Q}$.
Let $f: \mathcal{Q} \rightarrow \mathbb{R}^{(D+1) \times 3}$ denote the FM function mapping a CR configuration to its discretized backbone points in the CR base frame.
Given a base pose $\mathbf{T}_b \in SE(3)$, the corresponding world-frame backbone points are given by
$\mathbf{T}_b \cdot f(\mathbf{q})$.

An FTL motion plan is defined as a sequence of $N$ configurations
$\{(\mathbf{T}_{b,1}, \mathbf{q}_1), \ldots, (\mathbf{T}_{b,N}, \mathbf{q}_N)\} \subset \mathcal{K}$
that guides the robot through all waypoints in $\mathcal{W}$.
While there are $n$ waypoints, the full motion plan consists of $N \gg n$ configurations to ensure continuous interpolation between waypoints.
At each configuration index $k$, the planner must jointly choose the CR shape parameters $\mathbf{q}_k$ and the base pose $\mathbf{T}_{b,k}$.

As the robot progresses along the path, increasing portions of the CR backbone are considered “inserted”.
We associate each waypoint $\mathbf{w}_i$ with an insertion depth $s \in [0, S]$, representing the arc length along the backbone that should coincide with the path when the tip reaches $\mathbf{w}_i$.

We define two properties that should hold throughout the trajectory:
\begin{itemize}
    \item \textbf{Tip Tracking:} The tip should be placed at waypoint $\mathbf{w}_i$ when passing through it, and follow a smooth interpolation (e.g., linear) between consecutive waypoints. 
    \item \textbf{Shape Following:} The inserted portion of the robot should minimize deviation from the corresponding segment of the path.
\end{itemize}

A key distinction from classical FTL formulations is that the base pose $\mathbf{T}_b$ is not fixed.
Instead, the base pose is free to vary throughout the motion and is derived jointly with the CR configuration.
This additional freedom is essential for modeling CRs mounted on serial manipulators, but it fundamentally alters the FTL problem structure by removing the fixed reference frame typically assumed in prior work.

As a concrete instantiation, this work describes the method for a 3-segment TDCR with configuration space parameterized by Clarke coordinates~\cite{grassmann_2025_clarke_coordinates}, giving $\mathbf{q} \in \mathbb{R}^6$; the method is generalizable to any CR with any forward model.

\section{Method}

As discussed in Section~\ref{sec:intro}, the high-dimensional configuration space $\mathcal{K} = SE(3) \times \mathcal{Q}$ and nonlinear CR FM make iterative optimization approaches prone to local minima and computationally expensive. 
Our method addresses this by decoupling the problem into global shape search and local interpolation, yielding three key benefits:
\begin{itemize}
    \item \textbf{Model-agnostic:} FM evaluations occur primarily offline, enabling use of complex kinetostatic models that would be impractical for online optimization.
    \item \textbf{Efficient:} Online planning involves shape evaluation and closed-form geometric construction, with only small number of FM evaluations for interpolation.
    \item \textbf{Theoretically Guaranteed:} The sampling-based formulation provides formal guarantees including resolution completeness (Section~\ref{sec:theory}).
\end{itemize}
The method proceeds as follows: offline library generation (Section~\ref{subsec:shape_library}), online base pose alignment and shape search at each waypoint (Sections~\ref{subsec:base_align}--\ref{subsec:shape_eval}), and interpolation between waypoints (Section~\ref{subsec:interpolation}).

\begin{figure*}[!ht]
    \centering

    \includegraphics[height=5.5cm]{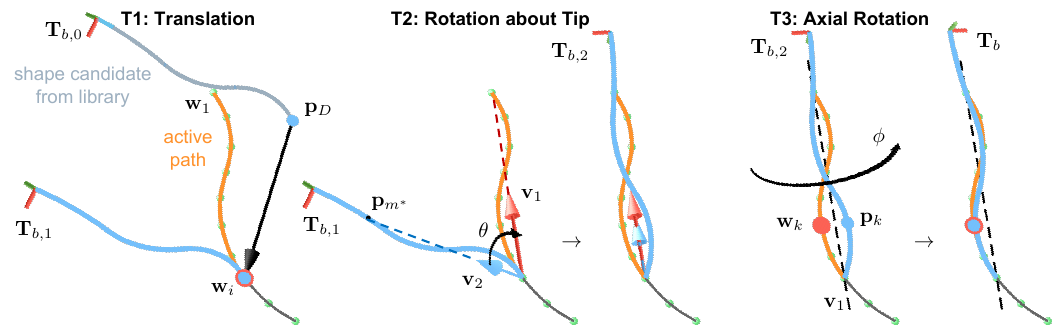}
    
    \caption{\textbf{Geometric base pose alignment} via three sequential transformations. The robot shape (blue) is aligned to the active path (orange) while guaranteeing exact tip placement at waypoint $\mathbf{w}_i$. \textbf{T1:} Translation by $(\mathbf{w}_i - \mathbf{p}_D)$ places the tip exactly at the waypoint. \textbf{T2:} Rotation about $\mathbf{w}_i$ by angle $\theta$ aligns the tip-to-robot direction $\mathbf{v}_2$ with the path direction $\mathbf{v}_1$. \textbf{T3:} Axial rotation about the $\mathbf{v}_1$ axis by angle $\phi$ resolves the remaining degree of freedom using a third correspondence pair $(\mathbf{w}_k, \mathbf{p}_k)$. Together, these transformations uniquely determine the 6-DOF base pose $\mathbf{T}_b$ with exact tip tracking by construction.
    }
    \label{fig:transformations}
\end{figure*}

\subsection{Shape Library}
\label{subsec:shape_library}

\subsubsection{Library Generation}
The shape library is constructed offline by sampling $N_{lib}$ configurations uniformly from the CR configuration space $\mathcal{Q}$. For each sampled configuration $\mathbf{q}_i$, we compute FM to obtain a discretized shape representation. Each shape is stored as a tuple $S^{(i)} = (\mathbf{q}^{(i)}, \mathbf{P}^{(i)})$, where $\mathbf{P}^{(i)} = [\mathbf{p}^{(i)}_{0}, \ldots, \mathbf{p}^{(i)}_{D}]^\top \in \mathbb{R}^{(D+1) \times 3}$ contains $D+1$ points along the backbone. $D$ should be larger than $n$ to meaningfully capture the CR shape. $D$ remains as a fixed constant for all shapes. The complete library is $\mathcal{L} = \{S^{(i)}\}_{i=1}^{N_{lib}}$. The shape library approximates the CR's reachable shape space. Since library lookup is much faster than forward model evaluation, this shifts the computational burden offline and enables efficient online planning.

\subsubsection{Threshold Clustering}
\label{subsubsec:clustering}
At each waypoint, we query $\mathcal{L}$ to find the shape minimizing deviation from the active path. The base case is linear search: evaluating all $N_{lib}$ shapes guarantees finding the optimal shape within the library, but incurs $\mathcal{O}(N_{lib})$ evaluations per waypoint.

To reduce computation, we introduce clustering as an optional mechanism that trades accuracy for speed. Shapes are grouped based on pairwise similarity:
\begin{equation}
    M(S^{(a)}, S^{(b)}) = \sum_{j=0}^{D} \|\mathbf{p}^{(a)}_j - \mathbf{p}^{(b)}_j\|
\end{equation}
Clusters are formed using a similarity threshold $\gamma$: iterating through ungrouped shapes, each becomes a cluster center and all shapes within distance $\gamma$ are added to its clusters, forming local balls of bounded radius in shape space.

The optimal $\gamma$ depends on the specific CR and its shape distribution. In practice, we empirically choose $\gamma$ by analyzing k-nearest neighbor distances in the library, targeting approximately $\lfloor\sqrt{N_{lib}} \cdot 1.5\rfloor$ clusters as a suggested default for the speed-accuracy tradeoff; see Appendix~\ref{app:clustering_ablation} for an ablation study on the effect of cluster size. Setting $\gamma = 0$ recovers the linear search behavior.

During online execution with clustering enabled, we first evaluate only cluster centers, then search exhaustively within the best-scoring cluster. This two-pass strategy reduces complexity from $\mathcal{O}(N_{lib})$ to $\mathcal{O}(\sqrt{N_{lib}})$, at the cost of potentially missing the globally optimal shape if it resides in a cluster whose center scores poorly.

\subsection{Base Pose Alignment}
\label{subsec:base_align}

At each waypoint $\textbf{w}_i$, we query the shape library to find the best-matching shape for the active path $\mathcal{W}_i = \{\textbf{w}_1, \ldots, \textbf{w}_i\}$. For each candidate shape $S = (\mathbf{q}, \mathbf{P})$, evaluation involves two steps: (1) align the shape to the path via a closed-form geometric construction that determines the base pose $\mathbf{T}_b$, and (2) compute the shape deviation after alignment. The shape with the minimum deviation is selected. Since alignment and evaluation are independent across shapes, the process is easily parallelizable. To align the shape to the path, the following steps are performed:

\subsubsection{Active Shape Subset}
Since the robot is only partially inserted at waypoint $\textbf{w}_i$, we first identify the active portion of the shape. Let $|\mathcal{W}_i| = \sum_{j=1}^{i-1} \|\textbf{w}_j - \textbf{w}_{j+1}\|$ denote the arc length of the active path. We select the subset of shape points from shape tip $\mathbf{p}_D$ whose arc length best matches $|\mathcal{W}_i|$:
\begin{equation}
    m^* = \arg\min_m \left| |\mathcal{W}_i| - \sum_{j=m}^{D-1} \|\mathbf{p}_j - \mathbf{p}_{j+1}\| \right|
\end{equation}
The active shape subset is $\mathbf{P}_{act} = \{\mathbf{p}_{m^*}, \ldots, \mathbf{p}_D\}$. Intuitively, we select the portion of the shape whose arc length best matches the inserted path length, ensuring meaningful shape-to-path comparison.

\subsubsection{Base Pose Alignment}
We compute the base pose $\mathbf{T}_b \in SE(3)$ through three sequential transformations that align the shape to the path with exact tip placement (Fig.~\ref{fig:transformations}).

\textbf{Transformation 1 (Translation):} Translate the shape by $(\textbf{w}_i - \mathbf{p}_D)$ so the tip coincides exactly with waypoint $\textbf{w}_i$. This ensures the tip-tracking property by construction.

\textbf{Transformation 2 (Rotation about tip):} Rotate the shape about $\textbf{w}_i$ to align the base-to-tip directions. Define unit vectors $\mathbf{v}_1 = (\textbf{w}_i - \textbf{w}_1)/\|\textbf{w}_i - \textbf{w}_1\|$ along the path and $\mathbf{v}_2 = (\textbf{w}_i - \mathbf{p}_{m^*})/\|\textbf{w}_i - \mathbf{p}_{m^*}\|$ along the translated shape. The rotation axis is $\mathbf{v} = \mathbf{v}_2 \times \mathbf{v}_1$ and angle is $\theta = \cos^{-1}(\mathbf{v}_1 \cdot \mathbf{v}_2)$. Since this rotation is about $\textbf{w}_i$, tip placement is preserved.

\textbf{Transformation 3 (Axial rotation):} A rotational degree of freedom remains about the path axis $\mathbf{v}_1$. To resolve this, we select a third correspondence point $\textbf{w}_k \in \mathcal{W}_i$ that maximizes distance to the line through $\textbf{w}_1$ and $\textbf{w}_i$, with corresponding shape point $\mathbf{p}_k$ matched by arc length ratio. The rotation angle $\phi$ is computed from the projections of $(\textbf{w}_k - \textbf{w}_1)$ and $(\mathbf{p}_k - \textbf{w}_1)$ onto the plane orthogonal to $\mathbf{v}_1$. Since, this rotation is also about $\mathbf{v}_1$ passing through $\textbf{w}_i$, tip placement is preserved.

Together, Transformations 1--3 uniquely determine all six degrees of freedom of $\mathbf{T}_b$ while guaranteeing $\mathbf{p}_D = \textbf{w}_i$ by construction. Note that this construction requires at least three non-collinear active waypoints. For the first two waypoints ($\textbf{w}_1$ and $\textbf{w}_2$), the base pose can be trivially determined using a home configuration (e.g., straight) or by reusing the shape found at $\textbf{w}_3$, depending on application requirements.

\subsection{Shape Evaluation}
\label{subsec:shape_eval}
After alignment, we compute the deviation between the transformed shape and the path using the symmetric Chamfer distance:
\begin{equation}\label{eqn:shape_metric}
\frac{1}{|\mathcal{W}_i|} \sum_{\mathbf{x} \in \mathcal{W}_i} \min_{\mathbf{y} \in \mathbf{P}_{act}} \|\mathbf{x} - \mathbf{y}\| + \frac{1}{|\mathbf{P}_{act}|} \sum_{\mathbf{y} \in \mathbf{P}_{act}} \min_{\mathbf{x} \in \mathcal{W}_i} \|\mathbf{y} - \mathbf{x}\|
\end{equation}
where $\mathbf{P}_{act}$ denotes active shape after alignment. We choose symmetric Chamfer distance as it is robust to differences in point density between the path and shape discretizations, and penalizes deviations bidirectionally. The shape minimizing the deviation cost is selected, and its associated $\mathbf{T}_b$ forms the complete configuration $(\mathbf{q}, \mathbf{T}_b) \in \mathcal{K}$.

This formulation can be viewed as nearest-neighbor search in a task-induced metric space: rather than measuring distance in configuration space as in classical sampling-based planners, we evaluate candidates by task performance (shape deviation) after analytically aligning the base pose.
Our formulation readily accommodates flexible cost design. For instance, different weights can be assigned to waypoints along the path, e.g., prioritizing recent waypoints as the robot progresses to emphasize accuracy near the tip. Additional terms such as base movement penalties, orientation change from the previous step, or collision constraints may also be incorporated to encourage smooth motion or address application-specific requirements.

\subsection{Interpolation}
\label{subsec:interpolation}

At this stage, we have a sparse motion plan with one configuration $(\mathbf{q}_i, \mathbf{T}_b^i)$ at each waypoint $\mathbf{w}_i$. To execute on real hardware, this sparse plan must be converted into a continuous trajectory in both CR joint space and base pose space. However, naive interpolation between consecutive configurations can result in large deviations from the FTL path, as the intermediate shapes may not align well with the path.

The main challenge is that consecutive configurations may differ significantly in both joint state and base pose (Fig.~\ref{fig:interpolationExm}), requiring large coordinated motions to maintain FTL path tracking. We address this in two steps: first, we exploit the radial symmetry of CRs to pre-align consecutive configurations, reducing orientation discontinuities; then, we interpolate smoothly in joint space while constructing base poses that guarantee exact tip placement throughout.

\subsubsection{Exploiting Radial Symmetry}
CRs exhibit radial symmetry about their backbone axis: rotating the robot about this axis produces a shape in a different bending plane without changing the shape itself. We exploit this to pre-align consecutive configurations before interpolation.

Given the sparse plan $\{(\mathbf{q}_j, \mathbf{T}_b^j)\}_{j=1}^{n}$, we rotate each configuration to maximize alignment with a common reference direction (e.g., the base $x$-axis of the first configuration, $x_1$). The optimal rotation angle $\theta_j$ is computed in closed form by maximizing the dot product between the rotated $x$-axis and the reference direction:
\begin{equation}
    \theta_j = \mathrm{atan2}\left(-\mathbf{x}_1^\top \mathbf{y}_j,\; \mathbf{x}_1^\top \mathbf{x}_j\right)
\end{equation}
where $\mathbf{x}_j$ and $\mathbf{y}_j$ are the $x$ and $y$ axes of the base orientation $\mathbf{T}_b^j$. Depending on the CR design, three types of radial symmetry determine how $\theta_j$ is applied:

\noindent\textbf{1. Continuous:} For CRs with fully symmetric actuation (e.g., PCC), $\theta_j$ is applied directly.

\noindent\textbf{2. Discrete:} For TDCRs with $k$ tendons per segment, $\theta_j$ is snapped to the nearest $360^\circ/k$ increment.

\noindent\textbf{3.~Data-driven:} When analytical symmetry properties are unknown, we identify radially symmetric shapes empirically. Using the same threshold-based clustering method from Section~\ref{subsec:shape_library}, we cluster shapes that have been rotated to align their tips to a common plane. Shapes within the same cluster are considered radially symmetric variants, allowing substitution with a similar shape with a more favorable base orientation, denoted $(\mathbf{q}_j', \mathbf{T}_b^{j\prime})$.

\begin{figure}[t]
    \centering
    \includegraphics[width=8.5cm]{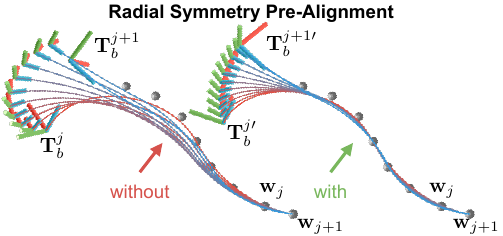}
    \caption{\textbf{Pre-alignment via radial symmetry.} Fanned frames show interpolated configurations between waypoints $\mathbf{w}_j$ (red shape) and $\mathbf{w}_{j+1}$ (blue shape). Left: without alignment, differing bending planes cause large base rotations during interpolation. Right: pre-aligning reduces orientation discontinuity while preserving shape, resulting in better shape following.}
    \label{fig:interpolationExm}
\end{figure}

\subsubsection{Local Interpolation Between Waypoints}
Between waypoints $\textbf{w}_j$ and $\textbf{w}_{j+1}$, we generate $h$ intermediate configurations to ensure smooth motion for practical deployment. The key idea is to define a smooth tip trajectory first, then compute the base pose that places the tip exactly on this trajectory.

We linearly interpolate the tip position and use SLERP~\cite{shoemake_1985_animating_rotation} for tip orientation, both obtained from the sparse plan, yielding a smooth desired tip pose ${}^k\!\mathbf{T}_{tip}$ at each interpolation step $k \in \{0, \ldots, h\}$ with $\alpha = k/h$:
\begin{align}
    {}^k\!\textbf{w} &= (1-\alpha) \textbf{w}_j + \alpha \textbf{w}_{j+1}\\
    {}^k\!\mathbf{R}_{tip} &= \text{SLERP}(\mathbf{R}_{tip}^{j\prime}, \mathbf{R}_{tip}^{(j+1)\prime}, \alpha)
\end{align}
Simultaneously, we linearly interpolate the CR joint state (after pre-alignment via radial symmetry):
\begin{equation}
    {}^k\!\mathbf{q}' = (1-\alpha) \mathbf{q}'_j + \alpha \mathbf{q}'_{j+1}
\end{equation}
Evaluating $f({}^k\!\mathbf{q}')$ gives the tip pose $\mathbf{T}_{tip}({}^k\!\mathbf{q}')$ in the CR's local frame. To place the tip exactly at the desired pose ${}^k\!\mathbf{T}_{tip}$, we construct the base pose as:
\begin{equation}
    {}^k\!\mathbf{T}_b = {}^k\!\mathbf{T}_{tip} \cdot \mathbf{T}_{tip}({}^k\!\mathbf{q}')^{-1}
\end{equation}
This construction guarantees exact tip tracking at every interpolation step. Each step requires one FM evaluation, and with $h$ steps between each of $n-1$ waypoint pairs, the complete motion plan consists of $N = (n-1) \cdot h + 1$ configurations with $(n-1) \cdot h$ total FM evaluations—a fixed, predictable cost independent of library size.

\section{Theoretical Guarantees}\label{sec:theory}

We analyze the theoretical properties of our motion planner, establishing resolution completeness and characterizing tip-tracking optimality. We present the analysis for the base case of linear search over the shape library. Section~\ref{sec:clustering_tradeoff} discusses the speed-accuracy tradeoff when using clustered search.

\subsection{Assumptions}

We require the following mild conditions:

\begin{itemize}
    \item[\textbf{A1}] The CR configuration space $\mathcal{Q}$ is compact.
    \item[\textbf{A2}] $C^1$ Forward Model: $f: \mathcal{Q} \rightarrow \mathcal{S}$ is continuously differentiable on $\mathcal{Q}$, where $\mathcal{S}$ is the space of robot shapes.
    \item[\textbf{A3}] Smooth Feasibility: For a feasible path $\mathcal{W}$, there exists a continuous function $\mathbf{q}^*: [0,1] \rightarrow \mathcal{Q}$ such that the corresponding shapes achieve bounded shape deviation at each insertion depth.
\end{itemize}

A1 holds for all practical robots with bounded joint limits. A2 is satisfied by standard CR models: PCC, Cosserat rod, FEM, and neural networks with smooth activations~\cite{rao_2020_how_to_model}. A2 further implies a deterministic forward model. A3 formalizes the existence of a smooth FTL motion and is standard in sampling-based planning.

\subsection{Resolution Completeness}
A motion planning algorithm is \textit{resolution complete} if, for a given discretization level, it is guaranteed to find a valid plan when one exists and terminate in finite time~\cite{barraquand_1991_robot_motion_planning}. 
In our framework, the discretization level corresponds to the discrete resolution library size $N_{lib}$, not the continuous configuration space. The key insight is that if a feasible configuration $\mathbf{q}^*$ exists, a sufficiently large library will contain a nearby configuration; Lipschitz continuity of shape deviation then ensures this neighbor achieves comparable performance, yielding a valid plan.
We formalize this below.

\begin{lemma}[Coverage Probability]
\label{lem:coverage}
Let $\mathcal{L}_N = \{\mathbf{q}_1, \ldots, \mathbf{q}_N\}$ be i.i.d.\ uniform samples from compact $\mathcal{Q}$. For any $\epsilon_C > 0$ and $\delta > 0$, there exists $N_0$ such that for $N \geq N_0$:
\begin{equation}
    \mathbb{P}[\mathcal{L}_N \text{ is an } \epsilon_C\text{-cover of } \mathcal{Q}] \geq 1 - \delta
\end{equation}
Intuitively, uniform sampling eventually places points densely throughout $\mathcal{Q}$, ensuring every configuration has a nearby sample. See Appendix~\ref{app:coverage_proof} for the full proof.

\end{lemma}

\begin{lemma}[Lipschitz Continuity of Shape Deviation]
\label{lem:lipschitz}
Under A1--A2, the shape deviation $E(S, \mathcal{W}_i)$ after optimal base pose alignment is Lipschitz continuous in the configuration:
\begin{equation}
    |E(\mathbf{q}_1, \mathcal{W}_i) - E(\mathbf{q}_2, \mathcal{W}_i)| \leq L_E \|\mathbf{q}_1 - \mathbf{q}_2\|
\end{equation}
for some constant $L_E < \infty$.
\end{lemma}

\begin{proof}
Under A1--A2, the forward model $f$ is continuously differentiable on compact $\mathcal{Q}$, so $\nabla f$ is bounded and $f$ is Lipschitz by the mean value theorem. The base pose construction (Transformations 1--3) depends continuously on the shape points, and shape deviation (\ref{eqn:shape_metric}) is continuous in the transformed shape. The composition of Lipschitz functions is Lipschitz. See Appendix~\ref{app:lipschitz_proof} for the full proof.
\end{proof}

\begin{theorem}[Resolution Completeness]
\label{thm:completeness}
Under A1--A3, for any $\delta > 0$, there exists library size $N_0$ such that for $N_{lib} \geq N_0$:
\begin{equation}
    \mathbb{P}[\text{valid plan found} \mid \text{feasible plan exists}] \geq 1 - \delta
\end{equation}
\end{theorem}

\begin{proof}
By A3, at each waypoint $\textbf{w}_i$ there exists $\mathbf{q}^*(i) \in \mathcal{Q}$ achieving shape deviation below some threshold $\epsilon^*$. By Lemma~\ref{lem:coverage}, for $N_{lib} \geq N_0$, with probability at least $1-\delta$, there exists $\hat{\mathbf{q}}_i \in \mathcal{L}_{N_{lib}}$ satisfying $\|\mathbf{q}^*(i) - \hat{\mathbf{q}}_i\| < \epsilon_C$. By Lemma~\ref{lem:lipschitz}:
\begin{equation}
    E(\hat{\mathbf{q}}_i, \mathcal{W}_i) \leq E(\mathbf{q}^*(i), \mathcal{W}_i) + L_E \epsilon_C < \epsilon^* + L_E \epsilon_C
\end{equation}
The planner selects the configuration minimizing $E(\cdot, \mathcal{W}_i)$, achieving deviation at most $\epsilon^* + L_E\epsilon_C$. As $N_{lib} \to \infty$, $\epsilon_C \to 0$, so the achieved deviation approaches $\epsilon^*$.
\end{proof}

\subsection{Tip-Tracking Convergence}

A central property of base pose alignment (Transformations 1--3) is that exact tip placement is guaranteed by geometric construction, independently of the selected shape and without iterative optimization:

\begin{theorem}[Exact Tip Tracking By Construction]
\label{thm:tip_exact}
For any selected configuration $\mathbf{q} \in \mathcal{Q}$:
\begin{enumerate}
    \item At each waypoint $\textbf{w}_i$, the tip position satisfies $\|\mathbf{p}_D - \textbf{w}_i\| = 0$.
    \item At each interpolation step $k$ between waypoints $\textbf{w}_j$ and $\textbf{w}_{j+1}$, the tip position satisfies $\mathbf{p}_D^{(k)} = (1-\alpha)\textbf{w}_j + \alpha \textbf{w}_{j+1}$ exactly, where $\alpha = k/h$.
\end{enumerate}
\end{theorem}

\begin{proof}
Property (1): Transformation 1 translates the shape by $(\textbf{w}_i - \mathbf{p}_D)$, placing the tip at $\textbf{w}_i$. Transformations 2--3 rotate about $\textbf{w}_i$, preserving this constraint. Property (2): The interpolated base pose is constructed as ${}^k\!\mathbf{T}_{tip} \cdot \mathbf{T}_{tip}({}^k\!\mathbf{q})^{-1}$ where ${}^k\!\mathbf{T}_{tip}$ has position component ${}^k\!\textbf{w} = (1-\alpha)\textbf{w}_j + \alpha \textbf{w}_{j+1}$, ensuring the tip matches the interpolated waypoint.
\end{proof}

While tip error is exactly zero at interpolation steps, hardware execution requires continuous joint motion between them. During this inter-step motion, the tip may deviate from the linear FTL path. We show this deviation is bounded and vanishes with finer interpolation:

\begin{theorem}[Asymptotic Tip-Tracking Convergence]
\label{thm:asymptotic}
Between consecutive interpolation steps, the tip deviation from the linearly-interpolated FTL path satisfies:
\begin{equation}
    \max_{\beta \in [0,1]} \|\mathbf{p}_{tip}(\beta) - \mathbf{p}_{ftl}(\beta)\| \leq \frac{C}{h^2} \Delta_{max}^2
\end{equation}
where $h$ is the number of interpolation steps, $\Delta_{max}$ bounds the configuration change between waypoints, and $C$ depends on FM smoothness. Consequently, as $h \rightarrow \infty$, the supremum tip error over the entire trajectory vanishes.
\end{theorem}

\begin{proof}
The tip error is zero at both endpoints of each inter-step interval (Theorem~\ref{thm:tip_exact}). Define the error $\mathbf{e}(\beta) = \mathbf{p}_{tip}(\beta) - \mathbf{p}_{ftl}(\beta)$. Since $\mathbf{e}(0) = \mathbf{e}(1) = \mathbf{0}$ and $\mathbf{e}$ is smooth (by A1--A2$'$), the maximum deviation satisfies $\max_\beta \|\mathbf{e}(\beta)\| \leq \frac{1}{8}\max_\beta \|\mathbf{e}''(\beta)\|$ (up to dimensional constants). Since $\|\mathbf{e}''(\beta)\|$ scales as $(\Delta_{max}/h)^2$, the $\mathcal{O}(1/h^2)$ bound follows. See Appendix~\ref{app:asymptotic_proof} for the full derivation.
\end{proof}

\subsection{Computational Tradeoff with Clustering}
\label{sec:clustering_tradeoff}

The theoretical guarantees established above require linear search over the full shape library $\mathcal{L}$, incurring $\mathcal{O}(N_{lib})$ shape evaluations per waypoint. The clustered variant (Section~\ref{subsubsec:clustering}) reduces this to approximately $2\sqrt{N_{lib}}$ evaluations, but sacrifices resolution completeness: if the optimal configuration resides in a cluster whose center scores poorly, that cluster will not be explored.

This presents a user-adjustable tradeoff: linear search retains full guarantees, while clustering enables faster planning at the cost of potential suboptimality. Crucially, tip-tracking guarantees (Theorems~\ref{thm:tip_exact}--\ref{thm:asymptotic}) remain valid regardless, as they depend only on base pose alignment.
In practice, we observe minimal increase in shape deviation with significant reduction in computation time (Section~\ref{sec:evaluation}), suggesting clustering is a viable choice for most applications.

\section{Evaluation}
\label{sec:evaluation}

\begin{table*}[]
\caption{Evaluation results across 120 paths with PCC. Our method outperforms the optimization-based baseline with exact tip tracking. The clustered variant maintains competitive shape deviation with significantly reduced computation time. Tip and shape deviation are reported as percentages of the total robot length (3 unit lengths).}
\label{tab:evaluation}
\begin{center}
\begin{tabular}{c|ccc|ccc|ccc}
 & \multicolumn{3}{c|}{\textbf{Tip Deviation (\%) $\downarrow$}} & \multicolumn{3}{c|}{\textbf{Shape Deviation (\%) $\downarrow$}} & \multicolumn{3}{c}{\textbf{Compute Time (s) $\downarrow$}} \\
\hline
Test Class & Optimization & Linear & Clustered & Optimization & Linear & Clustered & Optimization & Linear & Clustered \\
\hline
C Curves & 1.24 & \textbf{0.0} & \textbf{0.0} & 1.74 & \textbf{1.25} & 1.61 & 41.3 & 36.4 & \textbf{1.81} \\
S Curves & 1.21 & \textbf{0.0} & \textbf{0.0} & 2.36 & \textbf{1.94} & 2.09 & 39.3 & 46.7 & \textbf{2.35} \\
Robot Curves & 0.62 & \textbf{0.0} & \textbf{0.0} & 1.77 & \textbf{1.71} & 2.03 & 43.4 & 46.9 & \textbf{2.32}
\end{tabular}
\end{center}
\end{table*}

We evaluate our method in three settings of increasing fidelity: (1) comparison against an optimization-based baseline using a PCC kinematic model, (2) simulation with a system-identified pseudo-rigid-body model in MuJoCo, and (3) hardware demonstration on a physical tendon-driven CR mounted on a 7-DOF manipulator. Together, these experiments validate our method's accuracy, efficiency, and practical applicability.

\subsection{Metrics}
We evaluate performance using three metrics:
\begin{itemize}
    \item \textbf{Tip deviation}: distance between the CR tip and the target waypoint, normalized by robot length (\%)
    \item \textbf{Shape deviation}: symmetric Chamfer distance (Eqn.~\ref{eqn:shape_metric}) between the inserted robot shape and the active path, normalized by robot length (\%)
    \item \textbf{Compute time}: wall-clock time to generate a complete motion plan (seconds)
\end{itemize}

\subsection{PCC Model: Comparison with Optimization Baseline}
\label{subsec:pcc_eval}

\subsubsection{Setup}
We compare our method against an optimization-based FTL baseline using a piece-wise constant curvature (PCC) kinematic model. The baseline minimizes shape deviation using \texttt{SciPy}'s L-BFGS-B solver~\cite{2020SciPy-NMeth}, initialized with the solution from the previous waypoint or the zero configuration at the first waypoint. This sequential optimization approach aligns with prior FTL methods \cite{palmer_2014_realtime_method_tip, gao_2020_cr_ftl_edoscopic}. 

For our method, the shape library contains $N_{lib} = 20,000$ configurations uniformly sampled from $\mathcal{Q}$, chosen empirically to balance speed and accuracy. The CR has \num{3} segments of unit length each (total length \num{3}).(Section~\ref{subsubsec:shape_lib_size_effect} evaluates the effect of library size). We evaluate both the linear search variant and the clustered variant ($\gamma = 2.5$), with continuous radial symmetry for pre-alignment. All experiments were run on a 13th Gen Intel Core i7-13700F with 32GB memory.

\subsubsection{Test Paths}
\label{subsubsec:PCCsetupTestpaths}
We evaluate on 120 paths across three curve classes (40 per class):
\begin{itemize}
    \item \textbf{C Curves}: half circles with start at origin and endpoint sampled from $[0.5, 1.5] \times [-0.75, 0.25] \times [1, 2]$, with bending plane uniformly sampled from $[0, 2\pi]$
    \item \textbf{S Curves}: planar cubic Bézier curves with start at origin and endpoint sampled from $[-2.25, -1.25] \times [-0.5, 0.5] \times \{1.5\}$, with control points forming an S shape
    \item \textbf{Robot Curves}: 3D paths generated by FM on configurations sampled from a bounded subset of $\mathcal{Q}$, representing scenarios where the goal shape is known-feasible (e.g., retraction from a current configuration); they test planning quality rather than shape generalization.
\end{itemize}
Each path has $n=10$ waypoints with $h=10$ interpolation steps between consecutive waypoints.

\subsubsection{FTL Performance}
Table~\ref{tab:evaluation} summarizes the results. Our method achieves 0\% tip deviation across all test classes, confirming the exact tip tracking guarantee (Theorem~\ref{thm:tip_exact}). For shape deviation, the linear search variant outperforms the optimization baseline on all curve classes, confirming the hypothesis that sampling-based methods are better suited for FTL planning since they avoid local minima. Notably, the clustered variant reduces computation time by over an order of magnitude while sacrificing minimal shape accuracy.

\subsection{MuJoCo Simulation: Pseudo-Rigid-Body Model}
\label{subsec:mujoco_eval}

\begin{table}[b]
\caption{Evaluation results across 120 paths with MuJoCo simulation using a system-identified PRB model. Our method achieves exact tip tracking and reasonable shape deviation despite increased model complexity. Tip and shape deviation are reported as percentages of the total robot length.}
\label{tab:evaluationMJ}
\begin{center}
\begin{tabular}{c|c|c|c}
 & \multicolumn{1}{c|}{\textbf{Tip Dev. (\%) $\downarrow$}} & \multicolumn{1}{c|}{\textbf{Shape Dev. (\%) $\downarrow$}} & \multicolumn{1}{c}{\textbf{Time (s) $\downarrow$}} \\
\hline
Test Class & Clustered & Clustered & Clustered \\
\hline
C Curves & 0.0 & 2.46 & 2.07 \\
S Curves & 0.0 & 3.17 & 2.69 \\
Robot Curves & 0.0 & 2.43 & 2.85
\end{tabular}
\end{center}
\end{table}

\subsubsection{Setup}
We evaluate our method using a custom implementation of a pseudo-rigid-body (PRB) model in MuJoCo \cite{todorov_2012_mujoco}, where the CR is modeled by a series of 30 rigid links connected by elastic universal joints. System parameters such as joint stiffness are identified from a physical 3-segment tendon-driven CR to capture realistic deformation behavior, including segment coupling effects not modeled by PCC. 
The simulated CR is mounted on a 7-DOF manipulator (Franka Robotics) with a constant offset transformation $\mathbf{T}_{\text{offset}}$ accounting for the actuation unit and hardware mount. For each planned base pose $\mathbf{T}_b$, the manipulator end-effector pose is $\mathbf{T}_{ee} = \mathbf{T}_b \cdot \mathbf{T}_{\text{offset}}^{-1}$, which is converted to joint-level commands for the Franka robot using an inverse kinematics solver~\cite{corke_2021_rtb}. 

We test the clustered variant with shape library size of $N_{lib}=20,000$ and $\gamma=2.5$. 
We use the same three curve classes (C, S, Robot) with \num{40} paths per class, randomly initialized with the same distributions as in \ref{subsubsec:PCCsetupTestpaths}.

\subsubsection{Results}
Table~\ref{tab:evaluationMJ} summarizes the results. The method achieves 0\% tip deviation across all test classes. Shape deviation is slightly higher than the PCC experiments, reflecting the increased complexity of the PRB model with segment coupling and only three-fold discrete radial symmetry for pre-alignment. These results demonstrate that our method scales effectively to more realistic CR models: given any forward model, the planner generates motion plans with guaranteed tip tracking and competitive shape following.

\begin{figure}[htb]
    \centering
    \includegraphics[width=\columnwidth, height=0.7\columnwidth]{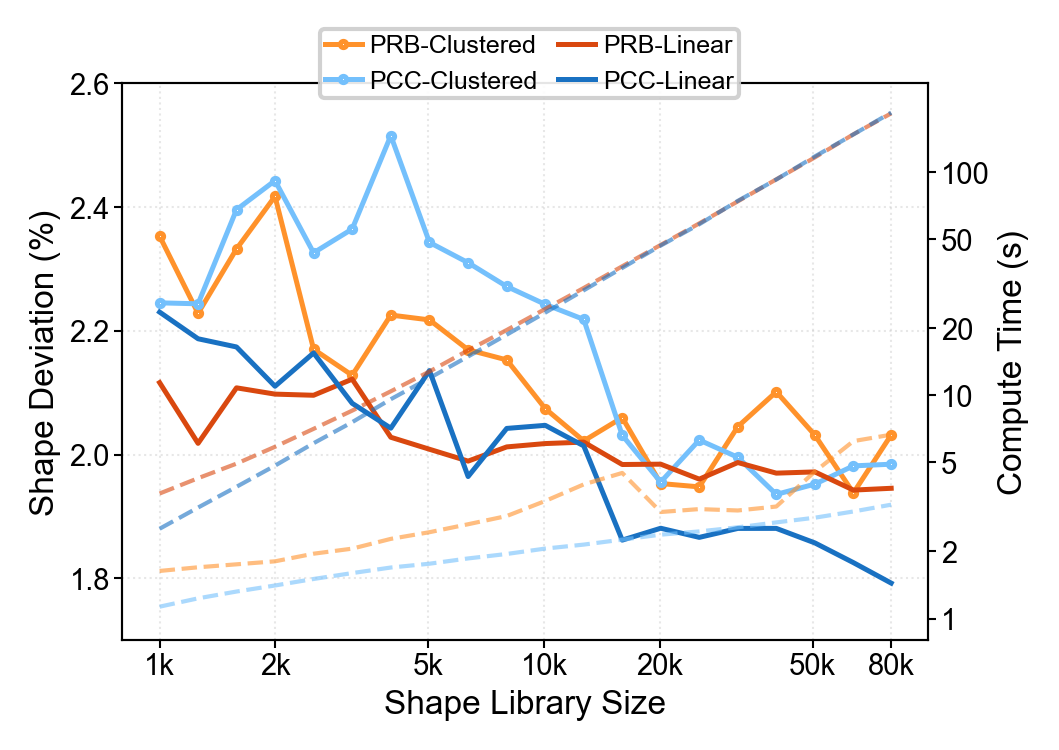}
    \caption{\textbf{Tradeoff between library size, accuracy, and computation time.} As $N_{lib}$ increases, shape deviation (solid lines) decreases with diminishing returns, while computation time (dashed lines) grows linearly for linear search. Clustering (light colors)} maintains computation time with small accuracy loss. Results shown for both PCC (blue) and PRB (orange).
    \label{fig:shapelibvstime}
\end{figure}

\subsubsection{Effect of Shape Library Size}
\label{subsubsec:shape_lib_size_effect}
Fig.~\ref{fig:shapelibvstime} illustrates the effect of the shape library size $N_{lib}$ on shape-following accuracy and computation time for both PCC and PRB models.
As $N_{lib}$ increases, shape deviation decreases with diminishing returns: rapid improvement is observed for small to moderate library sizes, followed by saturation as the library becomes sufficiently dense. 
This behavior is consistent with the resolution completeness analysis in Theorem~\ref{thm:completeness}, which predicts improved approximation of feasible shapes with increasing library resolution.
The computation time grows linearly with $N_{lib}$ for the linear variant. In contrast, the clustered variant maintains nearly constant computation time across library sizes while achieving comparable shape deviation, confirming the effectiveness of the two-pass search strategy (Section~\ref{sec:clustering_tradeoff}). The gap between PCC and PRB curves reflects the increased model complexity, but both exhibit the same qualitative trends. For time-critical applications, clustering provides an effective mechanism to maintain this accuracy with over an order of magnitude reduction in online computation.

\subsection{Hardware Demonstration}
\label{subsec:hardware_eval}

\subsubsection{Setup}
We demonstrate our method on a physical setup consisting of a 3-segment tendon-driven CR mounted on a 7-DOF manipulator (Franka Robotics). The CR has 3 tendons per segment with parallel tendon routing. Motion plans are generated using the MuJoCo PRB model from Section~\ref{subsec:mujoco_eval}.

\subsubsection{Execution}
The joint-level motion plan is post-processed using \texttt{toppra}~\cite{pham_2018_new_approach_time} for time-optimal path parameterization, reducing jerk while ensuring synchronized execution between the CR and manipulator. The resulting trajectories are sent to servo motors with built-in PID controllers (Dynamixel XL430-W250-T) for tendon actuation, and a joint-impedance controller for the serial manipulator. We executed representative paths from each of the three curve classes.

\subsubsection{Results}
Fig.~\ref{fig:hardware_setup} shows the hardware execution for C, S, and robot curve paths. The robot successfully tracks all paths, with the manipulator and CR moving in coordination throughout the trajectory. We measured tip tracking error using an optical tracker (Lyra, NDI, Canada), with a single marker attached to the CR tip and a four-marker rigid body defining the base frame relative to the manipulator base. Across the three paths, normalized by the CR length ($187.5$~mm), the C-shape achieved a mean error of $5.5\%$ (max $8.0\%$), the S-shape $9.1\%$ mean (max $20.6\%$), and the robot-curve shape $8.3\%$ mean (max $20.9\%$), with visible deviations between planned and executed shapes.

In simulation, the planner achieves zero tip error by construction, since waypoints are matched exactly against the same kinematic model used for execution. The hardware error therefore reflects the gap between this model and the physical system rather than a limitation of the planner itself. Unmodelled effects such as tendon stretch, friction, and hysteresis cause the physical robot to bend less than predicted, and because our pipeline executes plans open-loop from the system-identified model, these discrepancies accumulate without correction. The resulting hardware metrics are effectively a noisy version of the simulation results, with the noise attributable to system-specific factors (modeling fidelity, controller performance) that would not transfer across platforms. We therefore report these tip errors as a reference point for the specific hardware realization rather than as an evaluation of planner quality. The simulation experiments in Sections~\ref{subsec:pcc_eval}--\ref{subsec:mujoco_eval} isolate planner performance, while the hardware demonstration validates that our method produces smooth, executable trajectories that achieve qualitatively correct FTL motion on a physical manipulator-mounted CR system.

Closing this gap is orthogonal to the planning framework and can be approached from several directions: improving model fidelity through richer parametrizations of tendon-driven dynamics, more thorough system identification, generating the shape library directly from hardware-recorded shapes rather than from a model (which would absorb unmodelled effects into the library itself), or adding closed-loop shape control during execution. See the supplementary video for full execution sequences.

\begin{figure}[!t]
    \centering
    \includegraphics[width=6cm]{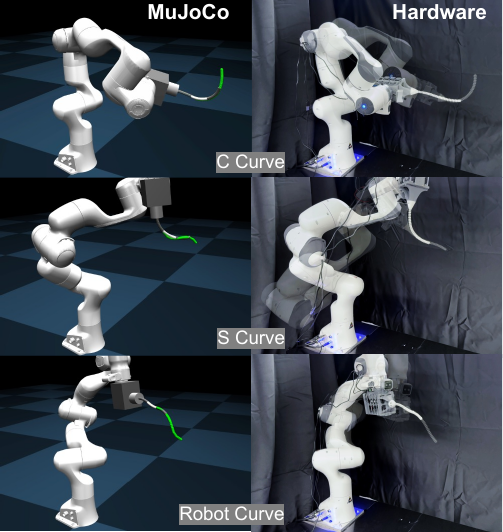}
    \caption{\textbf{Simulated (left) and hardware (right) evaluation} on three path classes. The setup consists of a 3-segment tendon-driven CR mounted on a 7-DOF manipulator. Hardware images overlay two frames: solid (mid-trajectory) and semi-transparent (final configuration), illustrating the coordinated base motion during FTL execution. See supplementary video for full execution sequences.}
    \label{fig:hardware_setup}
\end{figure}

\section{Conclusion}
\label{sec:conclusion}

This work introduced a sampling-based motion planning framework for follow-the-leader motion of manipulator-mounted continuum robots with full 6-DOF base pose control. By decoupling global shape search from local interpolation, the approach achieves exact tip tracking by design while supporting general forward models that would be impractical for iterative optimization. This decoupled structure also enables formal guarantees that optimization-based approaches cannot provide: resolution completeness from sampling-based search, exact tip tracking from geometric base pose construction, and asymptotic tip-tracking convergence from interpolation. User-adjustable clustering further provides a principled tradeoff between computation time and shape accuracy.

\subsection{Limitations and Future Directions}
\label{subsec:limitations_future}
The current work assumes the forward model is independent of base pose and external loading, whereas gravity can induce pose-dependent shape changes; pose-dependent shape prediction or active gravity compensation could address this. The per-step nature of shape search precludes trajectory-wide guarantees on shape-following quality; jointly optimizing across waypoints would provide stronger guarantees, though at the cost of substantially expanding the planning space. The uniform sampling strategy for library generation does not guarantee uniform coverage in shape space; importance sampling or adaptive refinement targeting shape space coverage could improve library quality for complex CRs without increasing library size. Finally, closed-loop shape control during execution, combined with more accurate forward models, would improve hardware execution fidelity.

More broadly, this work shows that FTL motion for manipulator-mounted CRs can be addressed without iterative optimization by decoupling shape search from base pose construction. We believe this opens the door to scalable FTL planning for increasingly complex CR systems, including hybrid rigid--continuum platforms.

\section*{Acknowledgments} 
We acknowledge the support of the Natural Sciences and Engineering Research Council of Canada (NSERC), [funding reference number RGPIN-2025-04343].
This research was also supported by the Canada Foundation for Innovation (CFI) through the John R. Evans Leaders Fund (JELF) [project number 40110].


\bibliographystyle{plainnat}
\bibliography{references}

\clearpage
\appendix
\section{Appendix}

\label{sec:appendix}

\setcounter{equation}{0}
\renewcommand{\theequation}{A\arabic{equation}}
\setcounter{figure}{0}
\renewcommand{\thefigure}{A\arabic{figure}}
\setcounter{table}{0}
\renewcommand{\thetable}{A.\Roman{table}}

In the Appendix, we first present ablation studies validating key design decisions (Section~\ref{app:ablation}), followed by full proofs for Lemma~\ref{lem:coverage} (Section~\ref{app:coverage_proof}), Lemma~\ref{lem:lipschitz} (Section~\ref{app:lipschitz_proof}), and Theorem~\ref{thm:asymptotic} (Section~\ref{app:asymptotic_proof}).

\subsection{Ablation Studies}
\label{app:ablation}
We present additional ablation studies to validate key design decisions of our framework. These are included in the appendix due to space constraints. All experiments use only the base Chamfer distance (Eqn.~\ref{eqn:shape_metric}) as the shape deviation metric. Tests are conducted on 45 paths (15 per class) initialized as in Section~\ref{subsubsec:PCCsetupTestpaths}.

\begin{figure}[!b]
    \centering
    \includegraphics[width=\columnwidth]{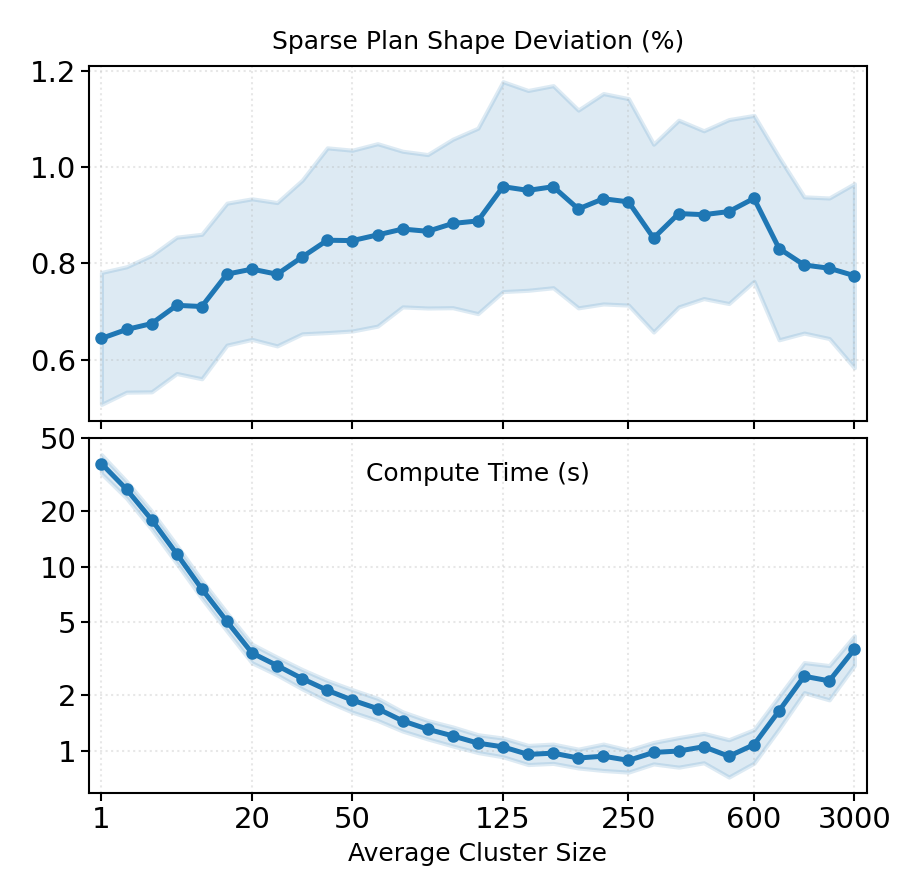}
    \caption{Effect of clustering on sparse plan shape deviation (top) and computation time (bottom) as a function of average cluster size. Smaller clusters approach linear search (lowest deviation, highest time); larger clusters reduce computation but increase deviation until very large clusters partially recover accuracy. Shaded regions indicate $\pm 1$ standard deviation.}
    \label{fig:cluster_ablation}
\end{figure}

\subsubsection{Effect of Clustering}
\label{app:clustering_ablation}
We evaluate the effect of clustering on shape-following accuracy and computation time. In this evaluation we use a 3-segment CR with PCC forward model and a library of $N_{lib} = 20{,}000$ shapes. To isolate the impact of shape search from interpolation effects, we measure shape deviation on the \emph{sparse plan only}. This is necessary because local interpolation introduces path-dependent variation: a suboptimal sparse plan can occasionally yield better interpolated trajectories.

Fig.~\ref{fig:cluster_ablation} shows sparse plan shape deviation and computation time as a function of average cluster size, which is controlled by the user-adjustable clustering threshold $\gamma$. A cluster size of 1 corresponds to $\gamma=0$, recovering linear search. As cluster size increases, shape deviation initially grows as expected: larger clusters increase the likelihood that the globally optimal shape resides in an unexplored cluster. At very large cluster sizes, deviation partially recovers as the library collapses into a small number of clusters, each spanning a substantial portion of the shape space. Computation time decreases rapidly with increasing cluster size but plateaus beyond moderate sizes. These results confirm that $\gamma$ provides an effective mechanism for trading accuracy against computation time based on application requirements.

\begin{figure}[!b]
    \centering
    \includegraphics[width=\columnwidth]{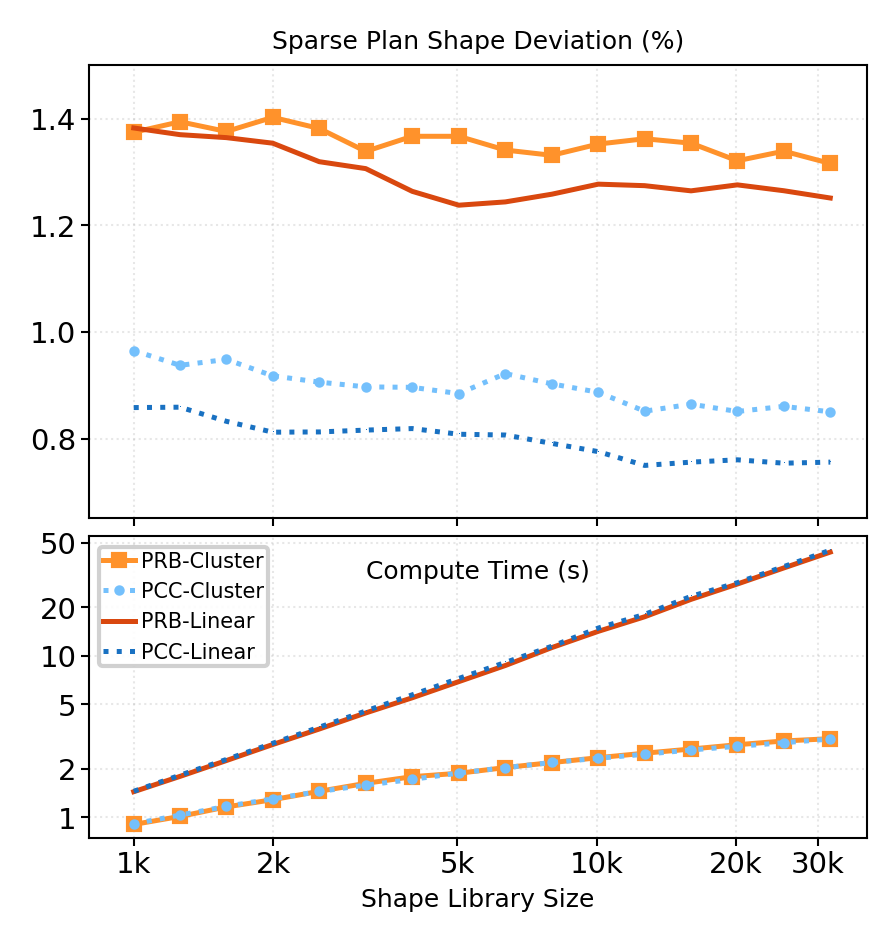}
    \caption{Effect of library size on sparse plan shape deviation and computation time. Without interpolation variance, trends are clearer: deviation decreases with diminishing returns, linear search time grows as $\mathcal{O}(N_{lib})$, and clustered search time grows as $\mathcal{O}(\sqrt{N_{lib}})$. PRB model shows higher deviation than PCC due to increased shape space complexity.}
    \label{fig:libsize_sparse_dev}
\end{figure}
\begin{table*}[!htb]
\caption{Effect of radial symmetry pre-alignment on overall shape deviation (\%). All variants use the same sparse plan; differences arise solely from interpolation quality.}
\label{tab:radial_ablation}
\begin{center}
\begin{tabular}{c|c|c|c}
& \multicolumn{3}{c}{\textbf{Shape Deviation (\%)} $\downarrow$} \\
\hline
Test Class 
& \multicolumn{1}{c|}{without Pre-alignment} 
 & \multicolumn{1}{c|}{with Discrete (3-Fold) Symmetry } 
 & \multicolumn{1}{c}{with Continuous (Full) Symmetry} \\
\hline
C Curves 
& 3.34 $\pm$ 1.80 
& 2.24 $\pm$ 0.76
& 2.10 $\pm$ 0.69 \\

S Curves 
& 3.08 $\pm$ 0.37 
& 2.26 $\pm$ 0.29 
& 2.19 $\pm$ 0.47 \\

Robot Curves 
& 2.24 $\pm$ 0.19 
& 2.04 $\pm$ 0.32 
& 1.85 $\pm$ 0.23 \\
\hline
Overall 
& 2.89 $\pm$ 1.11
& 2.18 $\pm$ 0.52
& 2.05 $\pm$ 0.52
\end{tabular}
\end{center}
\end{table*}

\subsubsection{Effect of Library Size on Sparse Plan}
\label{app:libsize_ablation}

The main paper (Fig.~\ref{fig:shapelibvstime}) reports overall deviation including interpolation, reflecting practical performance. Here, we measure sparse plan deviation only to isolate the effect of library size on shape search quality.

Fig.~\ref{fig:libsize_sparse_dev} shows shape deviation and computation time as a function of $N_{lib}$. Without interpolation noise, the trends are clearer: shape deviation decreases with diminishing returns as $N_{lib}$ increases. Computation time for linear search grows linearly ($\mathcal{O}(N_{lib})$), while the clustered variant follows approximately $\mathcal{O}(\sqrt{N_{lib}})$ growth with a modest accuracy tradeoff. The PRB model consistently exhibits higher shape deviation than PCC, reflecting its richer and more complex shape space. Both models exhibit the same qualitative trends, confirming the method's model-agnostic behavior.

\subsubsection{Effect of Radial Symmetry Pre-Alignment}
\label{app:radial_ablation}

We evaluate the effect of radial symmetry pre-alignment on overall shape deviation (including interpolation), comparing three configurations: no pre-alignment, discrete 3-fold symmetry (matching common tendon-driven CR designs), and continuous (full) symmetry. All three variants share the same sparse plan. All tests are conducted with a 3-segment CR with PCC forward model and a library of $N_{lib} = 20{,}000$ shapes.

Table~\ref{tab:radial_ablation} summarizes the results. Pre-alignment consistently improves shape deviation compared to no alignment, confirming that reducing base orientation differences between consecutive waypoints leads to smoother and more accurate interpolation. Full symmetry yields a modest further improvement over 3-fold, as it allows exact alignment to the reference direction rather than snapping to the nearest $120^\circ$ increment. The improvement from no alignment to 3-fold is substantially larger than from 3-fold to full, suggesting that even coarse symmetry exploitation captures the majority of the benefit. 

These results validate the use of radial symmetry as a pre-processing step for interpolation and suggest that for CRs with discrete tendon arrangements, discrete symmetry provides a practical approximation to full symmetry with minimal performance loss.

In summary, the ablation studies confirm that: (i) clustering 
provides a smooth tradeoff between speed and accuracy, (ii) library 
size follows diminishing returns consistent with resolution 
completeness, and (iii) radial symmetry pre-alignment 
significantly improves interpolation quality with even 
discrete symmetry capturing most of the benefit.

\subsection{Proof of Lemma~\ref{lem:coverage} (Coverage Probability)}
\label{app:coverage_proof}
\begin{proof}
Let $\mathcal{N}(\mathcal{Q}, \epsilon_C)$ denote the $\epsilon_C$-covering number of $\mathcal{Q}$, i.e., the minimum number of balls of radius $\epsilon_C$ needed to cover $\mathcal{Q}$. For compact $\mathcal{Q} \subset \mathbb{R}^d$:
\begin{equation}
    M := \mathcal{N}(\mathcal{Q}, \epsilon_C) \leq \left(\frac{\text{diam}(\mathcal{Q})}{\epsilon_C} + 1\right)^d
\end{equation}
Let $\{c_1, \ldots, c_M\}$ be the centers of a minimal $\epsilon_C$-cover, and define $B_j = B(c_j, \epsilon_C) \cap \mathcal{Q}$. Each $B_j$ has volume at least $V_{min} = c_d \epsilon_C^d$ where $c_d$ is a dimension-dependent constant.
For uniform sampling over $\mathcal{Q}$, the probability that a single sample falls in $B_j$ is:
\begin{equation}
    p_j = \frac{\text{Vol}(B_j)}{\text{Vol}(\mathcal{Q})} \geq p_{min} := \frac{c_d \epsilon_C^d}{\text{Vol}(\mathcal{Q})}
\end{equation}
The probability that none of $N$ i.i.d.\ samples falls in $B_j$ is $(1 - p_j)^N \leq (1 - p_{min})^N$. By the union bound:
\begin{equation}
    \mathbb{P}[\exists j: B_j \cap \mathcal{L}_N = \emptyset] \leq M(1 - p_{min})^N
\end{equation}
For this probability to be at most $\delta$, we require:
\begin{equation}
    M(1 - p_{min})^N \leq \delta
\end{equation}
Using $\ln(1-x) \leq -x$ for $x \in (0,1)$:
\begin{equation}
    N \geq \frac{\ln(M/\delta)}{p_{min}} = \frac{\text{Vol}(\mathcal{Q})}{c_d \epsilon_C^d} \ln\left(\frac{M}{\delta}\right)
\end{equation}
Setting $N_0 = \lceil \frac{\text{Vol}(\mathcal{Q})}{c_d \epsilon_C^d} \ln(\frac{M}{\delta}) \rceil$ completes the proof.
\end{proof}

\subsection{Expanded Proof of Lemma~\ref{lem:lipschitz} (Lipschitz Continuity of Shape Deviation)}
\label{app:lipschitz_proof}

\begin{proof}
We show that each component of the pipeline $\mathbf{q} \mapsto E(\mathbf{q}, \mathcal{W}_i)$ is Lipschitz continuous.

\textbf{Step 1: Forward model.} 
By A1--A2, the forward model $f: \mathcal{Q} \to \mathbb{R}^{(D+1)\times 3}$ is continuously differentiable on the compact domain $\mathcal{Q}$. Therefore, the gradient $\nabla f$ is continuous on $\mathcal{Q}$ and bounded by compactness, yielding Lipschitz continuity with constant $L_f = \sup_{\mathbf{q} \in \mathcal{Q}} \|\nabla f(\mathbf{q})\|$.

\textbf{Step 2: Base pose construction.} We show each transformation is Lipschitz in the shape points:
\begin{itemize}
    \item \textit{T1 (Translation):} The translation vector $(\mathbf{w}_i - \mathbf{p}_D)$ is linear in $\mathbf{p}_D$, hence Lipschitz with constant 1.
    
    \item \textit{T2 (Rotation about tip):} The rotation aligns direction $\mathbf{v}_2$ to $\mathbf{v}_1$ by rotating about axis $\mathbf{v} = \mathbf{v}_2 \times \mathbf{v}_1$ by angle $\theta = \cos^{-1}(\mathbf{v}_1 \cdot \mathbf{v}_2)$. We show each component is Lipschitz:

\begin{itemize}
    \item \textit{Unit vector $\mathbf{v}_2$:} $\mathbf{v}_2 = (\mathbf{w}_i - \mathbf{p}_{m^*})/\|\mathbf{w}_i - \mathbf{p}_{m^*}\|$. The normalization $\mathbf{x} \mapsto \mathbf{x}/\|\mathbf{x}\|$ has gradient $(\mathbf{I} - \hat{\mathbf{x}}\hat{\mathbf{x}}^\top)/\|\mathbf{x}\|$, which is bounded by $1/\|\mathbf{x}\|$. Thus normalization is Lipschitz with constant $1/\delta$ on the set $\{\mathbf{x} : \|\mathbf{x}\| \geq \delta\}$. Since $\mathbf{p}_{m^*}'$ is the base of the active shape (after T1 translation) and $\mathbf{w}_i$ is the tip waypoint, their distance equals the arc length of the active path, which is strictly positive. Hence $\delta > 0$ exists depending on path geometry.
    
    \item \textit{Rotation angle $\theta$:} $\theta = \cos^{-1}(\mathbf{v}_1 \cdot \mathbf{v}_2)$. The dot product $\mathbf{v}_1 \cdot \mathbf{v}_2$ is Lipschitz in $\mathbf{v}_2$ with constant $\|\mathbf{v}_1\| = 1$. The function $\cos^{-1}: [-1,1] \to [0,\pi]$ has derivative $-1/\sqrt{1-x^2}$, which is bounded on $[-1+\epsilon, 1-\epsilon]$ for any $\epsilon > 0$. This composition is Lipschitz except when $\mathbf{v}_1$ and $\mathbf{v}_2$ are exactly collinear, which does not occur for generic paths and shapes.
    
    \item \textit{Rotation axis $\mathbf{v}$:} $\mathbf{v} = \mathbf{v}_2 \times \mathbf{v}_1$. The cross product is bilinear, hence Lipschitz in $\mathbf{v}_2$ with constant $\|\mathbf{v}_1\| = 1$. Note that when $\theta \to 0$ or $\theta \to \pi$, $\|\mathbf{v}\| \to 0$, but the rotation itself becomes trivial (identity or $\pi$-rotation about any perpendicular axis), so the resulting pose varies continuously.
    
    \item \textit{Rotation matrix:} Given axis $\mathbf{v}$ and angle $\theta$, the Rodrigues formula gives $\mathbf{R} = \mathbf{I} + \sin\theta[\mathbf{v}]_\times + (1-\cos\theta)[\mathbf{v}]_\times^2$. This is smooth (hence Lipschitz) in both $\theta$ and $\mathbf{v}$ away from degeneracies.
\end{itemize}

Composing these, T2 is Lipschitz in the shape points on the set of non-degenerate configurations.
    
    \item \textit{T3 (Axial rotation):} The rotation resolves the remaining degree of freedom by aligning a third correspondence point. We show each component is Lipschitz:

\begin{itemize}
    \item \textit{Path axis $\mathbf{a}$:} $\mathbf{a} = (\mathbf{w}_i - \mathbf{w}_1)/\|\mathbf{w}_i - \mathbf{w}_1\|$ is fixed by the path geometry (independent of shape).
    
    \item \textit{Third correspondence point $\mathbf{p}_k$:} The shape point $\mathbf{p}_k$ is selected by arc length ratio, which is a continuous (linear interpolation) operation on the shape points. Hence $\mathbf{p}_k$ is Lipschitz in the shape.
    
    \item \textit{Projection onto orthogonal plane:} The projection $\mathbf{u} \mapsto \mathbf{u} - (\mathbf{u} \cdot \mathbf{a})\mathbf{a}$ is linear, hence Lipschitz with constant 1.
    
    \item \textit{Rotation angle $\phi$:} Let $\mathbf{u}_w = (\mathbf{w}_k - \mathbf{w}_1) - ((\mathbf{w}_k - \mathbf{w}_1) \cdot \mathbf{a})\mathbf{a}$ and $\mathbf{u}_p = (\mathbf{p}_k - \mathbf{w}_1) - ((\mathbf{p}_k - \mathbf{w}_1) \cdot \mathbf{a})\mathbf{a}$ be the projected vectors. The angle $\phi$ is computed via $\mathrm{atan2}$ from these projections. This is Lipschitz whenever both projected vectors have non-zero magnitude, which holds when $\mathbf{w}_k$ is not collinear with $\mathbf{w}_1$ and $\mathbf{w}_i$. By construction, $\mathbf{w}_k$ is chosen to maximize distance to the line through $\mathbf{w}_1$ and $\mathbf{w}_i$, ensuring non-collinearity for any path with at least three non-collinear waypoints.
    
    \item \textit{Rotation matrix:} As in T2, the Rodrigues formula yields a rotation matrix that is smooth in $\phi$ and the fixed axis $\mathbf{a}$.
\end{itemize}

Composing these, T3 is Lipschitz in the shape points for paths with at least three non-collinear waypoints.
\end{itemize}

\textbf{Step 3: Shape deviation.} The Chamfer distance (\ref{eqn:shape_metric}) involves nearest-neighbor queries:
\begin{equation}
    E = \frac{1}{|\mathcal{W}_i|} \sum_{\mathbf{x} \in \mathcal{W}_i} \min_{\mathbf{y} \in \mathbf{P}_{act}} \|\mathbf{x} - \mathbf{y}\| + \frac{1}{|\mathbf{P}_{act}|} \sum_{\mathbf{y} \in \mathbf{P}_{act}} \min_{\mathbf{x} \in \mathcal{W}_i} \|\mathbf{y} - \mathbf{x}\|
\end{equation}
The function $\mathbf{y} \mapsto \min_{\mathbf{x} \in \mathcal{W}_i} \|\mathbf{y} - \mathbf{x}\|$ is Lipschitz with constant 1 (distance to a fixed set). Summation and averaging preserve Lipschitz continuity.

\textbf{Step 4: Composition.} The composition of Lipschitz functions is Lipschitz. Therefore, $E(\mathbf{q}, \mathcal{W}_i)$ is Lipschitz in $\mathbf{q}$ with constant $L_E = L_f \cdot L_{T} \cdot 1$, where $L_{T}$ depends on the path geometry and is finite for all non-degenerate paths.

The Lipschitz continuity of the base pose construction holds under the following non-degeneracy conditions: (i) the active shape has non-zero arc length, (ii) the path contains at least three non-collinear waypoints, and (iii) the shape direction $\mathbf{v}_2$ is not exactly collinear with the path direction $\mathbf{v}_1$. These conditions are satisfied for all practical paths and generic shape library configurations.
\end{proof}

\subsection{Proof of Theorem~\ref{thm:asymptotic} (Asymptotic Tip-Tracking Convergence})
\label{app:asymptotic_proof}
\begin{proof}
Consider the interval between interpolation steps $k$ and $k+1$, parameterized by $\beta \in [0,1]$. Let:
\begin{itemize}
    \item $\mathbf{q}(\beta) = (1-\beta) \, {}^k\!\mathbf{q} + \beta \, {}^{k+1}\!\mathbf{q}$ be the linearly interpolated configuration
    \item $\mathbf{p}_{tip}(\beta)$ be the actual tip position after applying FM and base transformation
    \item $\mathbf{p}_{ftl}(\beta) = (1-\beta) \, {}^k\!\textbf{w} + \beta \, {}^{k+1}\!\textbf{w}$ be the interpolated desired tip trajectory
\end{itemize}
\textbf{Step 1: Boundary conditions.} By Theorem~\ref{thm:tip_exact}:
\begin{align}
    \mathbf{p}_{tip}(0) &= {}^k\!\textbf{w} = \mathbf{p}_{ftl}(0) \\
    \mathbf{p}_{tip}(1) &= {}^{k+1}\!\textbf{w} = \mathbf{p}_{ftl}(1)
\end{align}
\textbf{Step 2: Smoothness.} Since FM is continuously differentiable (A2) on compact $\mathcal{Q}$ (A1), $\nabla f$ is bounded and $f$ is Lipschitz. The base pose construction is continuous in the shape. Therefore $\mathbf{p}_{tip}(\beta)$ is a smooth function of $\beta$.

\noindent\textbf{Step 3: Taylor expansion.} Define the error 
$\mathbf{e}(\beta) = \mathbf{p}_{tip}(\beta) - \mathbf{p}_{ftl}(\beta)$. 
We have $\mathbf{e}(0) = \mathbf{e}(1) = \mathbf{0}$. 
For smooth scalar functions with homogeneous boundary 
conditions on $[0,1]$, the maximum deviation satisfies 
$\max |g| \leq \frac{1}{8} \max |g''|$. Applying this 
component-wise to $\mathbf{e}(\beta) \in \mathbb{R}^3$ and 
bounding each $|e_i''(\beta)| \leq \|\mathbf{e}''(\beta)\|$ yields:
\begin{equation}
    \max_{\beta \in [0,1]} \|\mathbf{e}(\beta)\| \leq 
    \frac{\sqrt{3}}{8} \max_{\beta} \|\mathbf{e}''(\beta)\|
\end{equation}
The $\sqrt{3}$ factor is absorbed into the constant $C$ in the 
final bound.

\textbf{Step 4: Bounding the second derivative.} The second derivative $\mathbf{e}''(\beta)$ depends on:
\begin{itemize}
    \item The Hessian of FM with respect to configuration
    \item The rate of change of base pose with respect to shape
\end{itemize}
Both are bounded by constants depending on FM smoothness. The configuration change per interpolation step is $\|{}^{k+1}\!\mathbf{q} - {}^k\!\mathbf{q}\| \leq \Delta_{max}/h$ where $\Delta_{max} = \max_j \|\mathbf{q}_{j+1} - \mathbf{q}_j\|$ bounds the change between waypoints.
Since $\mathbf{e}''$ scales with the square of configuration velocity:
\begin{equation}
    \|\mathbf{e}''(\beta)\| \leq C' \left(\frac{\Delta_{max}}{h}\right)^2
\end{equation}
\textbf{Step 5: Final bound.} Combining:
\begin{equation}
    \max_{\beta} \|\mathbf{e}(\beta)\| \leq \frac{\sqrt{3}C'}{8} \cdot \frac{\Delta_{max}^2}{h^2} = \frac{C}{h^2} \Delta_{max}^2
\end{equation}
where $C = \sqrt{3}C'/8$ absorbs the constants.
\end{proof}

\end{document}